\documentclass{article}

\usepackage{microtype}
\usepackage{graphicx}
\usepackage{subfigure}
\usepackage{booktabs} 
\usepackage[scaled=0.85]{beramono}
\usepackage{amsmath}
\usepackage{amssymb,amsthm,bbm,graphicx,placeins,enumitem,dblfloatfix}
\usepackage{thm-restate}
\usepackage{pdfpages}

\DeclareMathOperator*{\argmax}{arg\,max}
\DeclareMathOperator*{\argmin}{arg\,min}
\newcommand{\<}{\mspace{1.5mu}}

\def\Rmax{R_{\text{\textnormal{max}}}}

\def\Vmax{V_{\text{\textnormal{max}}}}

\def\num{\rho_1}
\def\nuv{\rho_0}
\def\sm{s}
\def\sv{s_0}

\def\Mv{M_0}

\newtheorem{theorem}{Theorem}

\newtheorem{lemma}{Lemma}

\newtheorem{assumption}{Assumption}

\theoremstyle{definition}

\usepackage{hyperref}

\usepackage[accepted]{icml2018}

\icmltitlerunning{Feedback-Based Tree Search for Reinforcement Learning}

\begin{document}

\twocolumn[
\icmltitle{Feedback-Based Tree Search for Reinforcement Learning}




\begin{icmlauthorlist}
\icmlauthor{Daniel R. Jiang}{pitt}
\icmlauthor{Emmanuel Ekwedike}{tc,pton}
\icmlauthor{Han Liu}{tc,nw}
\end{icmlauthorlist}

\icmlaffiliation{pitt}{University of Pittsburgh}
\icmlaffiliation{tc}{Tencent AI Lab}
\icmlaffiliation{pton}{Princeton University}
\icmlaffiliation{nw}{Northwestern University}

\icmlcorrespondingauthor{Daniel R. Jiang}{drjiang@pitt.edu}

\icmlkeywords{MCTS, reinforcement learning, tree search}

\vskip 0.3in
]



\printAffiliationsAndNotice{}  

\begin{abstract}
Inspired by recent successes of Monte-Carlo tree search (MCTS) in a number of artificial intelligence (AI) application domains, we propose a model-based reinforcement learning (RL) technique that iteratively applies MCTS on batches of small, finite-horizon versions of the original infinite-horizon Markov decision process. The terminal condition of the finite-horizon problems, or the \emph{leaf-node evaluator} of the decision tree generated by MCTS, is specified using a combination of an estimated value function and an estimated policy function. The recommendations generated by the MCTS procedure are then provided as feedback in order to refine, through classification and regression, the leaf-node evaluator for the next iteration. We provide the first sample complexity bounds for a tree search-based RL algorithm. In addition, we show that a deep neural network implementation of the technique can create a competitive AI agent for the popular multi-player online battle arena (MOBA) game \textit{King of Glory}. 
\end{abstract}
\section{Introduction}
Monte-Carlo tree search (MCTS), introduced in \citet{Coulom2006} and surveyed in detail by \citet{Browne2012}, has received attention in recent years for its successes in gameplay artificial intelligence (AI), culminating in the Go-playing AI AlphaGo \citep{Silver2016}. MCTS seeks to iteratively build the decision tree associated with a given Markov decision process (MDP) so that attention is focused on ``important'' areas of the state space, assuming a given initial state (or \textit{root node} of the decision tree). The intuition behind MCTS is that if rough estimates of state or action values are given, then it is only necessary to expand the decision tree in the direction of states and actions with high estimated value. To accomplish this, MCTS utilizes the guidance of \emph{leaf-node evaluators} (either a policy function \cite{Chaslot2006} rollout, a value function evaluation \cite{Campbell2002,Enzenberger2003}, or a mixture of both \citep{Silver2016}) to produce estimates of downstream values once the tree has reached a certain depth \citep{Browne2012}. The information from the leaf-nodes are then \emph{backpropagated} up the tree. The performance of MCTS depends heavily on the quality of the policy/value approximations \citep{Gelly2007}, and at the same time, the successes of MCTS in Go show that MCTS \emph{improves upon a given policy} when the policy is used for leaf evaluation, and in fact, it can be viewed as a policy improvement operator \citep{silver2017mastering}. In this paper, we study a new feedback-based framework, wherein MCTS updates its own leaf-node evaluators using observations generated at the root node.

MCTS is typically viewed as an \emph{online planner}, where a decision tree is built starting from the current state as the root node  \citep{Chaslot2006,Chaslot2008,Hingston2007,Maitrepierre2008,Cazenave2009,Mehat2010,Gelly2011,Gelly2012,Silver2016}. The standard goal of MCTS is to recommend an action for the \emph{root node only}. After the action is taken, the system moves forward and a new tree is created from the next state (statistics from the old tree may be partially saved or completely discarded). MCTS is thus a ``local'' procedure (in that it only returns an action for a given state) and is inherently different from value function approximation or policy function approximation approaches where a ``global'' policy (one that contains policy information about all states) is built. In real-time decision-making applications, it is more difficult to build an adequate ``on-the-fly'' local approximation than it is to use pre-trained global policy in the short amount of time available for decision-making. For games like Chess or Go, online planning using MCTS may be appropriate, but in games where fast decisions are necessary (e.g., Atari or MOBA video games), tree search methods are too slow \citep{Guo2014}. The proposed algorithm is intended to be used in an \emph{off-policy} fashion during the reinforcement learning (RL) \emph{training phase}. Once the training is complete, the policies associated with leaf-node evaluation can be implemented to make fast, real-time decisions without any further need for tree search.

\noindent \textbf{Main Contributions.} These characteristics of MCTS motivate our proposed method, which attempts to leverage the \emph{local} properties of MCTS into a training procedure to iteratively build \emph{global} policy across all states. The idea is to apply MCTS on batches of small, finite-horizon versions of the original infinite-horizon Markov decision process (MDP). A rough summary is as follows: (1) initialize an arbitrary value function and a policy function; (2) start (possibly in parallel) a batch of MCTS instances, limited in search-depth, initialized from a set of sampled states, while incorporating a combination of the value and policy function as leaf-node evaluators; (3) update both the value and policy functions using the latest MCTS root node observations; (4) Repeat starting from step (2). This method exploits the idea that an MCTS policy is better than either of the leaf-node evaluator policies alone \citep{Silver2016}, yet improved leaf-node evaluators also improve the quality of MCTS \citep{Gelly2007}. The primary contributions of this paper are summarized below.
\begin{enumerate}
\item We propose a batch, MCTS-based RL method that operates on continuous state, finite action MDPs and exploits the idea that leaf-evaluators can be updated to produce a stronger tree search using \emph{previous tree search results}. Function approximators are used to track policy and value function approximations, where the latter is used to reduce the length of the tree search rollout (oftentimes, the rollout of the policy becomes a computational bottle-neck in complex environments).
\item We provide a full sample complexity analysis of the method and show that with large enough sample sizes and sufficiently large tree search effort, the performance of the estimated policies can be made close to optimal, up to some unavoidable approximation error. To our knowledge, batch MCTS-based RL methods have not been theoretically analyzed.
\item An implementation of the feedback-based tree search algorithm using deep neural networks is tested on the recently popular MOBA game $\mathtt{King\; of\; Glory}$ (a North American version of the same game is titled $\mathtt{Arena \; of \; Valor}$). The result is a competitive AI agent for the 1v1 mode of the game.
\end{enumerate}

\section{Related Work} The idea of leveraging tree search during training was first explored by \citet{Guo2014} in the context of Atari games, where MCTS was used to generate offline training data for a supervised learning (classification) procedure. The authors showed that by using the power of tree search offline, the resulting policy was able to outperform the deep $Q$-network (DQN) approach of \cite{Mnih2013}. A natural next step is to repeatedly apply the procedure of \citet{Guo2014}. In building AlphaGo Zero, \citet{silver2017mastering} extends the ideas of \citet{Guo2014} into an iterative procedure, where the neural network policy is updated after every episode and then reincorporated into tree search. The technique was able to produce a superhuman Go-playing AI (and improves upon the previous AlphaGo versions) without any human replay data. 

Our proposed algorithm is a \emph{provably near-optimal} variant (and in some respects, generalization) of the AlphaGo Zero algorithm. The key differences are the following: (1) our theoretical results cover a continuous, rather than finite, state space setting, (2) the environment is a stochastic MDP rather than a sequential deterministic two player game, (3) we use batch updates, (4) the feedback of previous results to the leaf-evaluator manifests as both policy and value updates rather than just the value (as \citet{silver2017mastering} does not use policy rollouts). 

\citet{anthony2017thinking} proposes a general framework called \emph{expert iteration} that combines supervised learning with tree search-based planning. The methods described in \citet{Guo2014}, \citet{silver2017mastering}, and the current paper can all be (at least loosely) expressed under the expert iteration framework. However, no theoretical insights were given in any of these previous works and our paper intends to fill this gap by providing a full theoretical analysis of an iterative, MCTS-based RL algorithm. Our analysis relies on the \emph{concentrability coefficient} idea of \citet{munos2007performance} for approximate value iteration and builds upon the work on classification based policy iteration \citep{lazaric2016analysis}, approximate modified policy iteration \citep{scherrer2015approximate}, and fitted value iteration \citep{munos2008finite}.

Sample complexity results for MCTS are relatively sparse. \citet{teraoka2014efficient} gives a high probability upper bound on the number of playouts needed to achieve $\epsilon$-accuracy at the root node for a stylized version of MCTS called $\mathtt{FindTopWinner}$. More recently, \citet{kaufmann2017monte} provided high probability bounds on the sample complexity of two other variants of MCTS called $\mathtt{UGapE}$-$\mathtt{MCTS}$ and $\mathtt{LUCB}$-$\mathtt{MCTS}$. In this paper, we do not require any particular implementation of MCTS, but make a generic assumption on its accuracy that is inspired by these results.

\section{Problem Formulation}
Consider a discounted, infinite-horizon MDP with a continuous state space $\mathcal S$ and finite action space $\mathcal A$. For all $(s,a) \in \mathcal S \times \mathcal A$, the \emph{reward function} $r: \mathcal S \times \mathcal A \rightarrow \mathbb R$ satisfies $r(s, a) \in [0,\Rmax]$. The \emph{transition kernel}, which describes transitions to the next state given current state $s$ and action $a$, is written $p(\<\cdot\<|\<s,a)$ --- a probability measure over $\mathcal S$. Given a \emph{discount factor} $\gamma \in [0,1)$, the value function $V^\pi$ of a \emph{policy} $\pi : \mathcal S \rightarrow \mathcal A$ starting in $s = s_0 \in \mathcal S$ is given by
\begin{equation}
V^\pi(s) = \mathbf{E} \left[ \sum_{t = 0}^\infty \gamma^t \, r(s_t, \pi_t(s_t))\right],
\label{eq:discountedpi}
\end{equation}
where $s_t$ is the state visited at time $t$. Let $\Pi$ be the set of all stationary, deterministic policies (i.e., mappings from state to action). The \emph{optimal value function} is obtained by maximizing over all policies: $V^*(s) = \sup_{\pi \in \Pi} V^\pi(s)$.

Both $V^\pi$ and $V^*$ are bounded by $V_\text{max} = \Rmax/(1-\gamma)$. We let $\mathcal F$ be the set of bounded, real-valued functions mapping $\mathcal S$ to $[0,V_\text{max}]$. We frequently make use of the shorthand operator $T_\pi : \mathcal F \to \mathcal F$, where the quantity $(T_\pi V)(s)$ is be interpreted as the reward gained by taking an action according to $\pi$, receiving the reward $r(s,\pi(s))$, and then receiving an expected terminal reward according to the argument $V$:
\begin{align*}
(T_\pi V)(s) = r(s,\pi(s)) + \gamma \int_{\mathcal S}  V(\tilde{s}) \,p(d\tilde{s}\<|\<s,\pi(s)).
\end{align*}
It is well-known that $V^\pi$ is the unique fixed-point of $T_\pi$, meaning $T_\pi V^\pi = V^\pi$ \citep{Puterman}. The \emph{Bellman operator} $T:\mathcal F \rightarrow \mathcal F$ is similarly defined using the maximizing action:
\[
(T V)(s) = \max_{a \in \mathcal A}\,\Bigl[ r(s,a) + \gamma \int_{\mathcal S} V(\tilde{s}) \,p(d\tilde{s}\<|\<s,a)  \Bigr].
\]
It is also known that $V^*$ is the unique fixed-point of $T$ \citep{Puterman} and that acting greedily with respect to the optimal value function $V^*$ produces an \emph{optimal policy}:
\[
\pi^*(s) \in \argmax_{a \in A}  \,\Bigl[r(s,a) + \gamma \int_{\mathcal S} V^*(\tilde{s}) \, p(d\tilde{s}\<|\<s,a)  \Bigr].
\]
We use the notation $T^d$ to mean the $d$ compositions of the mapping $T$, e.g., $T^2V = T(TV)$.   Lastly, let $V \in \mathcal F$ and let $\nu$ be a distribution over $\mathcal S$. We define left and right versions of an operator $P_\pi$:
\begin{align*}
(P_\pi V)(s) &= \int_{\mathcal S} V(\tilde{s}) \, p(d\tilde{s}\< | \< s, \pi(s)),\\
(\nu P_\pi)(d\tilde{s}) &= \int_{\mathcal S} p(d\tilde{s}\< | \< s, \pi(s)) \, \nu(ds).
\end{align*}
Note that $P_\pi V \in \mathcal F$ and $\mu P_\pi$ is another distribution over $\mathcal S$.

\section{Feedback-Based Tree Search Algorithm}
We now formally describe the proposed algorithm. The parameters are as follows. Let $\bar{\Pi} \subseteq \Pi$ be a space of approximate policies and $\bar{\mathcal F} \subseteq \mathcal F$ be a space of approximate value functions (e.g., classes of neural network architectures). We let $\pi_k \in \bar{\Pi}$ be the policy function approximation (PFA) and $V_k \in \bar{\mathcal F}$ be the value function approximation (VFA) at iteration $k$ of the algorithm. Parameters subscripted with `0' are used in the value function approximation (regression) phase and parameters subscripted with `1' are used in the tree search phase.
The full description of the procedure is given in Figure \ref{alg:pd}, using the notation $T_a = T_{\pi_a}$, where $\pi_a$ maps all states to the action $a \in \mathcal A$. We now summarize the two phases, VFA (Steps 2 and 3) and MCTS (Steps 4, 5, and 6).

\noindent \textbf{VFA Phase.} Given a policy $\pi_{k}$, we wish to approximate its value by fitting a function using subroutine $\mathtt{Regress}$ on $N_0$ states sampled from a distribution $\rho_0$. Each call to $\mathtt{MCTS}$ requires repeatedly performing rollouts that are initiated from leaf-nodes of the decision tree. Because repeating full rollouts during tree search is expensive, the idea is that a VFA obtained from a one-time regression on a single set of rollouts can drastically reduce the computation needed for $\mathtt{MCTS}$. For each sampled state $s$, we estimate its value using $M_0$ full rollouts, which can be obtained using the absorption time formulation of an infinite horizon MDP \citep[Proposition 5.3.1]{Puterman}.

\noindent \textbf{MCTS Phase.} On every iteration $k$, we sample a set of $N_1$ i.i.d. states from a distribution $\rho_1$ over $\mathcal S$. From each state, a tree search algorithm, denoted $\mathtt{MCTS}$, is executed for $M_1$ iterations on a search tree of maximum depth $d$. We assume here that the leaf evaluator is a general function of the PFA and VFA from the previous iteration, $\pi_{k}$ and $V_{k}$, and it is denoted as a ``subroutine'' $\mathtt{LeafEval}$. The results of the $\mathtt{MCTS}$ procedure are piped into a subroutine $\mathtt{Classify}$, which fits a new policy $\pi_{k+1}$ using classification (from continuous states to discrete actions) on the new data. As discussed more in Assumption \ref{as:class}, $\mathtt{Classify}$ uses $L_1$ observations (one-step rollouts) to compute a loss function.
\begin{figure}[h]
\centering{\framebox[1.078\width]{\parbox{3in}{
\begin{enumerate}[wide, labelwidth=!, labelindent=0pt, itemsep=1pt]
\small
\item Sample a set of $N_0$ i.i.d. states $\mathcal S_{0,\<k}$ from $\rho_0$ and $N_1$ i.i.d. states $\mathcal S_{1,k}$ from $\rho_1$.
\item Compute a sample average $\hat{Y}_{k}(s)$ of $M_0$ independent rollouts of $\pi_{k}$ for each  $s \in {\mathcal S}_{0,k}$. See Assumption \ref{as:reg}.
\item Use $\mathtt{Regress}$ on the set $\{ \hat{Y}_{k}(s) : s \in  {\mathcal S}_{0,k} \}$ to obtain a value function $V_{k} \in \bar{\mathcal F}$. See Assumption \ref{as:reg}.
\item From each $s \in \mathcal S_{1,k}$, run $\mathtt{MCTS}$ with parameters $M_1$, $d$, and evaluator $\mathtt{LeafEval}$. Return estimated value of each $s$, denoted $\hat{U}_k(s)$. See Assumptions \ref{as:leafeval} and \ref{as:treesearch}.
\item For each $s \in \mathcal S_{1,k}$ and $a \in \mathcal A$, create estimate $\hat{Q}_k(s,a) \approx (T_a \< V_k)(s)$ by averaging $L_1$ transitions from $p(\<\cdot\<|\<s,a)$. See Assumption \ref{as:class}.
\item Use $\mathtt{Classify}$ to solve a cost-sensitive classification problem and obtain the next policy $\pi_{k+1} \in \bar{\Pi}$. Costs are measured using $\{ \hat{U}_k(s) : s \in {\mathcal S}_{1,k} \}$ and $\{ \hat{Q}_k(s,\pi_{k+1}(s)): s \in \mathcal S_{1,k} \}$. See Assumption \ref{as:class}. Increment $k$ and return to Step 1.
\end{enumerate}
}}}
\caption{Feedback-Based Tree Search Algorithm}
\label{alg:pd}
\end{figure}

\begin{figure}
\centering
	\includegraphics[scale=1.0]{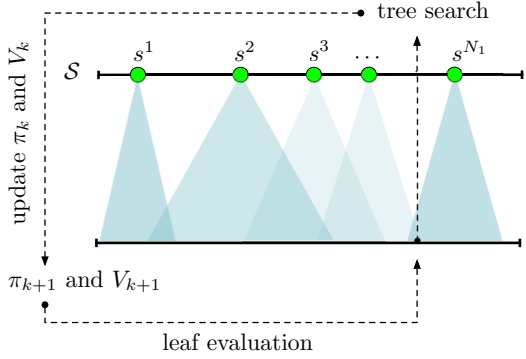}
	\vspace{-10pt}
	\caption{Illustration of the Feedback Loop}
	\label{fig:overview}
\end{figure}

The illustration given in Figure \ref{fig:overview} shows the interactions (and feedback loop) of the basic components of the algorithm: (1) a set of tree search runs initiated from a batch of sampled states (triangles), (2) leaf evaluation using $\pi_k$ and $V_k$ is used during tree search, and (3) updated PFA and VFA $\pi_{k+1}$ and $V_{k+1}$ using tree search results.

\section{Assumptions}

Figure \ref{alg:pd} shows the algorithm written with general subroutines $\mathtt{Regress}$, $\mathtt{MCTS}$, $\mathtt{LeafEval}$, and $\mathtt{Classify}$, allowing for variations in implementation suited for different problems. However, our analysis assumes specific choices and properties of these subroutines, which we describe now. The regression step solves a least absolute deviation problem to minimize an empirical version of
\[
\|f - V^{\pi_{k}}\|_{1,\,\nuv} = \int_{\mathcal S}|f(s) - V^{\pi_{k}}(s)| \< \rho_0(ds),
\]
as described in the first assumption.
\begin{assumption}[$\mathtt{Regress}$ Subroutine] 
\label{as:reg}
For each $s^i \in \mathcal S_{0,k}$, define $s^i = s_0^{ij}$ for all $j$ and for each $t$, the state $s_{t+1}^{ij}$ is drawn from $p( \, \cdot \< | \< s_t^{ij}, \pi_k(s_t^{ij}))$. Let $\hat{Y}_{k}(s^i)$ be an estimate of $V^{\pi_{k}}(s^i)$ using $\Mv$ rollouts and $V_{k}$, the VFA resulting from $\mathtt{Regress}$, obtained via least absolute deviation regression:
\begin{align}
\hat{Y}_{k}(\sv^i) = \frac{1}{\Mv} \sum_{j=1}^{\Mv} \sum_{t=0}^\infty \gamma^t \, r(s_t^{ij}, \pi_{k}(s_t^{ij}))\label{eq:Y},\\
V_{k} \in \argmin_{f \in \bar{\mathcal F}} \, \frac{1}{N_0} \sum_{i = 1}^{N_0} \< \bigl| f(s^i) - \hat{Y}_{k}(s^i) \bigr|\label{eq:V}.
\end{align}
\end{assumption}

There are many ways that $\mathtt{LeafEval}$ may be defined. The standard leaf evaluator for MCTS is to simulate a default or ``rollout'' policy \citep{Browne2012} until the end of the game, though in related tree search techniques, authors have also opted for a value function approximation \citep{Campbell2002,Enzenberger2003}. It is also possible to combine the two approximations: \citet{Silver2016} uses a weighted combination of a full rollout from a pre-trained policy and a pre-trained value function approximation. 

\begin{assumption}[$\mathtt{LeafEval}$ Subroutine]
\label{as:leafeval}
Our approach uses a partial rollout of length $h \ge0$ and a value estimation at the end. $\mathtt{LeafEval}$ produces unbiased observations of
\begin{equation}
J_{k}(s) = \mathbf{E} \left[ \sum_{t=0}^{h-1} \gamma^t \< r(\tilde{s}_t, \pi_{k}(\tilde{s}_t)) + \gamma^h \, V_{k}(\tilde{s}_h) \right],
\label{eq:u}
\end{equation}
where $\tilde{s}_0 = s$.
\end{assumption}

Assumption \ref{as:leafeval} is motivated by our MOBA game, on which we observed that even short rollouts (as opposed to simply using a VFA) are immensely helpful in determining local outcomes (e.g., dodging attacks, eliminating minions, health regeneration). At the same time, we found that numerous full rollouts simulated using the relatively slow and complex game engine is far too time-consuming within tree search.

We also need to make an assumption on the sample complexity of $\mathtt{MCTS}$, of which there are many possible variations \citep{Chaslot2006,Coulom2006,Kocsis2006,Gelly2007,couetoux2011continuous,couetoux2011continuous2,Al-Kanj2016,Jiang2017}. Particularly relevant to our continuous-state setting are tree expansion techniques called \emph{progressive widening} and \emph{double progressive widening}, proposed in \citet{couetoux2011continuous}, which have proven successful in problems with continuous state/action spaces. To our knowledge, analysis of the sample complexity is only available for stylized versions of MCTS on finite problems, like \citet{teraoka2014efficient} and \citet{kaufmann2017monte}. Theorems from these papers show upper bounds on the number of iterations needed so that with high probability (greater than $1-\delta$), the value at the root node is accurate within a tolerance of $\epsilon$. Fortunately, there are ways to discretize continuous state MDPs that enjoy error guarantees, such as \citet{bertsekas1975convergence}, \citet{dufour2012approximation}, or \citet{saldi2017asymptotic}. These error bounds can be combined with the MCTS guarantees of \citet{teraoka2014efficient} and \citet{kaufmann2017monte} to produce a sample complexity bound for MCTS on continuous problems. The next assumption captures the essence of these results (and if desired, can be made precise for specific implementations through the references above).

\begin{assumption}[$\mathtt{MCTS}$ Subroutine]
\label{as:treesearch}
Consider a $d$-stage, finite-horizon subproblem of (\ref{eq:discountedpi}) with terminal value function $J$ and initial state is $s$. Let the result of $\mathtt{MCTS}$ be denoted $\hat{U}(s)$. We assume that there exists a function $m(\epsilon,\delta)$, such that if $m(\epsilon,\delta)$ iterations of $\mathtt{MCTS}$ are used, the inequality $| \hat{U}(s) - (T^{\<d} \< J)(s) | \le \epsilon$ holds with probability at least $1- \delta$.
\end{assumption}
Now, we are ready to discuss the $\mathtt{Classify}$ subroutine. Our goal is to select a policy $\pi \in \bar{\Pi}$ that \emph{closely mimics} the performance of the $\mathtt{MCTS}$ result, similar to practical implementations in existing work \citep{Guo2014,silver2017mastering,anthony2017thinking}. The question is: given a candidate $\pi$, how do we measure ``closeness'' to the $\mathtt{MCTS}$ policy? We take inspiration from previous work in classification-based RL and use a cost-based penalization of classification errors \citep{langford2005relating,li2007focus,lazaric2016analysis}. Since $\hat{U}(s^i)$ is an approximation of the performance of the $\mathtt{MCTS}$  policy, we should try to select a policy $\pi$ with similar performance. To estimate the performance of some candidate policy $\pi$, we use a one-step rollout and evaluate the downstream cost using $V_k$.

\begin{assumption}[$\mathtt{Classify}$ Subroutine]
\label{as:class}
For each $s^i \in \mathcal S_{1,k}$ and $a \in \mathcal A$, let $\hat{Q}_k(\sm^i,a)$ be an estimate of the value of state-action pair $(s^i,a)$ using $L_1$ samples. 
\begin{align*}
\hat{Q}_k(\sm^i,a) =  \frac{1}{L_1}\sum_{j=1}^{L_1}\, \bigl[ r(\sm^i, a) + \gamma \< V_{k}(\tilde{s}^j(a))\bigr].
\end{align*}
Let $\pi_{k+1}$, the result of $\mathtt{Classify}$, be obtained by minimizing the discrepancy between the $\mathtt{MCTS}$ result $\hat{U}_k$ and the estimated value of the policy under approximations $\hat{Q}_k$:
\begin{align*}
\pi_{k+1} \in \argmin_{\pi \in \bar{\Pi}} \, \frac{1}{N_1} \sum_{i = 1}^{N_1} \, \bigl| \hat{U}_k(s^i) - \hat{Q}_k(s^i,\pi(\sm^i)) \bigr|,
\end{align*}
where $\tilde{s}^j(a)$ are i.i.d. samples from  $p( \< \cdot \, |\,\sm^i, a)$. 
\end{assumption}

An issue that arises during the analysis is that even though we can control the distribution from which states are sampled, this distribution is transformed by the transition kernel of the policies used for rollout/lookahead. Let us now introduce the \emph{concentrability coefficient} idea of \citet{munos2007performance} (and used subsequently by many authors, including \citet{munos2008finite}, \citet{lazaric2016analysis}, \citet{scherrer2015approximate}, and \citet{haskell2016empirical}).

\begin{assumption}[Concentrability]
Consider any sequence of $m$ policies $\mu_1, \mu_2, \ldots, \mu_m \in \Pi$. Suppose we start in distribution $\nu$ and that the state distribution attained after applying the $m$ policies in succession, $\nu \< P_{\mu_1} \< P_{\mu_2} \cdots P_{\mu_m}$, is absolutely continuous with respect to $\rho_1$. We define an $m$-step concentrability coefficient
\begin{align*}
A_m = \sup_{\mu_1, \< \ldots ,\< \mu_m}  \left \| \frac{d \nu P_{\mu_1} P_{\mu_2} \cdots P_{\mu_m}}{d\rho_1}  \right \|_{\infty},
\end{align*}
and assume that $\sum_{i,j=0}^\infty \gamma^{i+j} \< A_{i+j} < \infty$. Similarly, we assume $\num \< P_{\mu_1} \< P_{\mu_2} \cdots P_{\mu_m}$, is absolutely continuous with respect to $\nuv$ and assume that
\begin{align*}
A'_m &= \sup_{\mu_1, \< \ldots,\< \mu_m}  \left \| \frac{d \num P_{\mu_1} P_{\mu_2} \cdots P_{\mu_m}}{d\nuv}  \right \|_{\infty}
\end{align*}
is finite for any $m$.
\label{as:conc}
\end{assumption}
The concentrability coefficient describes how the state distribution changes after $m$ steps of arbitrary policies and how it relates to a given reference distribution. Assumptions \ref{as:reg}-\ref{as:conc} are used for the remainder of the paper.

\section{Sample Complexity Analysis}
Before presenting the sample complexity analysis, let us consider an algorithm that generates a sequence of policies $\{\pi_0, \pi_1, \pi_2, \ldots\}$ satisfying $T_{\pi_{k+1}} T^{d-1} \< V^{\pi_k} = T^d \< V^{\pi_k}$ with no error. It is proved in \citet[pp. 30-31]{Bertsekas1996} that $\pi_k \rightarrow \pi^*$ in the finite state and action setting. Our proposed algorithm in Figure \ref{alg:pd} can be viewed as \emph{approximately satisfying} this iteration in a continuous state space setting, where $\mathtt{MCTS}$ plays the role of $T^d$ and evaluation of $\pi_k$ uses a combination of accurate rollouts (due to $\mathtt{Classify}$) and fast VFA evaluations (due to $\mathtt{Regress}$). The sample complexity analysis requires the effects of all errors to be systematically analyzed.

For some $K \ge 0$, our goal is to develop a high probability upper bound on the \emph{expected suboptimality, over an initial state distribution} $\nu$, of the performance of policy $\pi_K$, written as $\|V^* - V^{\pi_K}\|_{1,\< \nu}$. Because there is no requirement to control errors with probability one, bounds in $\| \cdot\|_{1,\<\nu}$ tend to be much more useful in practice than ones in the traditional $\|\cdot \|_\infty$. 
Notice that:
\begin{equation}
\begin{aligned}
\frac{1}{N_1} \sum_{i = 1}^{N_1} \, \bigl| \hat{U}_k(s^i) &- \hat{Q}_k(s^i,\pi_{k+1}(s^i)) \bigr| \\
&\approx \bigl \|\<T^d \< V^{\pi_{k}} - T_{\pi_{k+1}} \< V^{\pi_{k}} \bigr \|_{1,\<\rho_1},
\end{aligned}
\label{eq:loss}
\end{equation}
where the left-hand-side is the loss function used in the classification step from Assumption \ref{as:class}. It turns out that we can relate the right-hand-side (albeit under a different distribution) to the expected suboptimality after $K$ iterations $\|V^* - V^{\pi_K}\|_{1,\< \nu}$, as shown in the following lemma. Full proofs of all results are given in the supplementary material.

\begin{restatable}[Loss to Performance Relationship]{lemma}{lemone}
The expected suboptimality of $\pi_K$ can be bounded as follows:
\begin{align*}
\|V^* - &V^{\pi_K}\|_{1,\< \nu} \le \; \gamma^{K\<d} \, \|V^* - V^{\pi_0}\|_{\infty} \\
&+ \sum_{k=1}^{K} \gamma^{(K-k)\<d} \, \bigl \|\<T^d \< V^{\pi_{k-1}} - T_{\pi_k} \< V^{\pi_{k-1}} \bigr \|_{1,\<\Lambda_{\nu,k}}
\end{align*}
where $\Lambda_{\nu,k} = \nu \, (P_{\pi^*})^{(K-k)\<d} \, \bigl[I - (\gamma P_{\pi_{k}})\bigr]^{-1}$.
\label{lem:one}
\end{restatable}

\begin{figure*}[h]
\centering
	\includegraphics[scale=0.95]{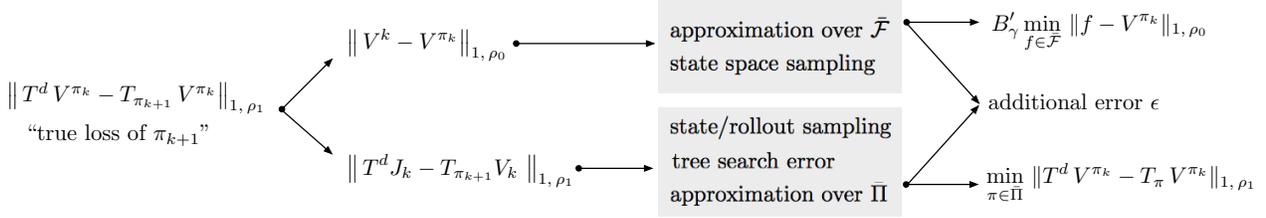}
	\vspace{-5pt}
	\caption{Various Errors Analyzed in Lemma \ref{lem:loss}\vspace{-10pt}}
	\label{fig:distances}
\end{figure*}

From Lemma \ref{lem:one}, we see that the expected suboptimality at iteration $K$ can be upper bounded by the suboptimality of the initial policy $\pi_0$ (in maximum norm) plus a discounted and re-weighted version of $\|\<T^d \< V^{\pi_{k-1}} - T_{\pi_k} \< V^{\pi_{k-1}} \|_{1,\<\rho_1}$ accumulated over prior iterations. Hypothetically, if $(T^d \< V^{\pi_{k-1}})(s) - (T_{\pi_k} \< V^{\pi_{k-1}})(s)$ were small for all iterations $k$ and all states $s$, then the suboptimality of $\pi_K$ converges linearly to zero. Hence, we may refer to $\|\<T^d \< V^{\pi_{k-1}} - T_{\pi_k} \< V^{\pi_{k-1}} \|_{1,\<\rho_1}$ as the ``true loss,'' the target term to be minimized at iteration $k$. We now have a starting point for the analysis: if (\ref{eq:loss}) can be made precise, then the result can be combined with Lemma \ref{lem:one} to provide an explicit bound on $\|V^* - V^{\pi_K}\|_{1,\nu}$. The various errors that we incur when relating the objective of $\mathtt{Classify}$ to the true loss include the error due to regression using functions in $\bar{\mathcal F}$; the error due to sampling the state space according to $\rho_1$; the error of estimating $(T_\pi\<V_k)(s)$ using the sample average of one-step rollouts $\hat{Q}_k(s,\pi(s))$; and of course, the error due to $\mathtt{MCTS}$.

We now give a series of lemmas that help us carry out the analysis. In the algorithmic setting, the policy $\pi_k$ is a random quantity that depends on the samples collected in previous iterations; however, for simplicity, the lemmas that follow are stated from the perspective of a fixed policy $\mu$ or fixed value function approximation $V$ rather than $\pi_k$ or $V_k$. Conditioning arguments will be used when invoking these lemmas (see supplementary material).
\begin{lemma}[Propagation of VFA Error]
\label{lem:Atwo}
Consider a policy $\mu \in \Pi$ and value function $V \in \mathcal F$. Analogous to (\ref{eq:u}), let $J = T_\mu^h \< V$. Then, under Assumption \ref{as:conc}, we have the bounds:
\begin{enumerate}[label=(\alph*)]
\item $\sup_{\pi \in \bar{\Pi}} \| T_{\pi} \< V - T_{\pi} \< V^{\mu}  \|_{1,\<\num} \le  \gamma \< A'_{1} \, \| V-V^\mu \|_{1,\<\nuv}$,
\item $\| T^d \< J -T^d \< V^\mu  \|_{1,\<\num} \le \gamma^{d+h}\< A'_{d+h} \< \| V-V^\mu \|_{1,\<\nuv}$.
\end{enumerate}
\end{lemma}
The lemma above addresses the fact that instead of using $V^{\pi_k}$ directly, $\mathtt{Classify}$ and $\mathtt{MCTS}$ only have access to the estimates $V_k$ and $J_k = T^h_{\pi_k} V_k$ ($h$ steps of rollout with an evaluation of $V_k$ at the end), respectively. Note that propagation of the error in $V_k$ is discounted by $\gamma$ or $\gamma^{d+h}$ and since the lemma converts between $\|\cdot\|_{1,\<\rho_1}$ and $\|\cdot\|_{1,\<\rho_0}$, it is also impacted by the concentrability coefficients $A'_1$ and $A'_{d+h}$.

Let $\textnormal{d}_{\bar{\Pi}}$ be the \emph{VC-dimension} of the class of binary classifiers $\bar{\Pi}$ and let $\textnormal{d}_{\bar{\mathcal F}}$ be the \emph{pseudo-dimension} of the function class $\bar{\mathcal F}$. The VC-dimension is a measure of the \emph{capacity} of $\bar{\Pi}$ and the notion of a pseudo-dimension is a generalization of the VC-dimension to real-valued functions (see, e.g., \citet{pollard1990empirical}, \citet{haussler1992decision}, \citet{mohri2012foundations} for definitions of both). Similar to \citet{lazaric2016analysis} and \citet{scherrer2015approximate}, we will present results for the case of two actions, i.e., $|\mathcal A| = 2$. The extension to multiple actions is possible by performing an analysis along the lines of \citet[Section 6]{lazaric2016analysis}. 
We now quantify the error illustrated in Figure \ref{fig:distances}. Define the quantity $B'_\gamma = \gamma \< A'_{1} + \gamma^{d+h}\< A'_{d+h}$, the sum of the coefficients from Lemma \ref{lem:Atwo}.
\begin{lemma}
\label{lem:loss}
Suppose the regression sample size $N_0$ is
\[
\mathcal O \left( (\Vmax B'_\gamma)^2 \, \epsilon^{-2} \, \bigl [ \log(1/\delta) + d_{\bar{\mathcal F}} \< \log(\Vmax \< B'_\gamma/\epsilon)   \bigr] \right)
\]
and the sample size $M_0$, for estimating the regression targets, is
\[
\mathcal O\left( (\Vmax B'_\gamma)^2 \, \epsilon^{-2} \, \bigl [ \log(N_0/\delta) \bigr] \right).
\]
Furthermore, there exist constants $C_1$, $C_2$, $C_3$, and $C_4$, such that if $N_1$ and $L_1$ are large enough to satisfy
\begin{align*}
&N_1 \ge C_1 \< \Vmax^2 \, \epsilon^{-2} \, \bigl[   \log(C_2/\delta) + d_{\bar{\Pi}} \< \log(e \< N_1/d_{\bar{\Pi}} )   \bigr],\\
&L_1 \ge C_1 \< \Vmax^2 \, \epsilon^{-2} \, \bigl[   \log(C_2\<N_1/\delta) + d_{\bar{\Pi}} \< \log(e \< L_1/d_{\bar{\Pi}} )   \bigr],
\end{align*}
and if $M_1 \ge m(C_3\<\epsilon, C_4 \< \delta/N_1)$, then
\begin{align*}
\|T^d\< V^{\pi_k} - T_{\pi_{k+1}} &\< V^{\pi_k} \|_{1,\<\rho_1} \le B_\gamma' \min_{f \in \bar{\mathcal F}} \< \| f-V^{\pi_k} \|_{1,\<\nuv} \\
&+ \min_{\pi \in \bar{\Pi}} \, \|T^d\< V^{\pi_k} - T_\pi \< V^{\pi_k} \|_{1,\<\rho_1}  + \epsilon
\end{align*} 
with probability at least $1-\delta$.
\end{lemma}
\begin{proof}[Sketch of Proof]
By adding and subtracting terms, applying the triangle inequality, and invoking Lemma \ref{lem:Atwo}, we see that:
\begin{align*}
\|T^d\< V^{\pi_k} - T_{\pi_{k+1}} &\< V^{\pi_k} \|_{1,\<\rho_1} \le  B_\gamma' \, \| V_k - V^{\pi_k} \|_{1,\<\rho_0}\\
&+ \|T^d J_k - T_{\pi_{k+1}} V_k \|_{1, \< \rho_1},\label{thm:eqmain}
\end{align*}
Here, the error is split into two terms. The first depends on the sample $\mathcal S_{0,k}$ and the history through $\pi_k$ while the second term depends on the sample $\mathcal S_{1,k}$ and the history through $V_k$. We can thus view $\pi_k$ as fixed when analyzing the first term and $V_k$ as fixed when analyzing the second term (details in the supplementary material). The first term $\| V_k - V^{\pi_k} \|_{1,\<\rho_0}$ contributes the quantity $\min_{f \in \bar{\mathcal F}} \< \| f-V^{\pi_k} \|_{1,\<\nuv}$ in the final bound with additional estimation error contained within $\epsilon$. The second term $ \|T^d J_k - T_{\pi_{k+1}} V_k \|_{1, \< \rho_1}$ contributes the rest. See Figure \ref{fig:distances} for an illustration of the main proof steps.
\end{proof}
The first two terms on the right-hand-side are related to the approximation power of $\bar{\mathcal F}$ and $\bar{\Pi}$ and can be considered unavoidable. We upper-bound these terms by maximizing over $\bar{\Pi}$, in effect removing the dependence on the random process $\pi_k$ in the analysis of the next theorem. We define:
\begin{align*}
&\mathbb D_0(\bar{\Pi},\bar{\mathcal F}) = \max_{\pi \in \bar{\Pi}} \,\min_{f \in \bar{\mathcal F}} \< \| f-V^{\pi} \|_{1,\<\nuv},\\
&\mathbb D_1^d(\bar{\Pi}) = \max_{\pi \in \bar{\Pi}} \,\min_{\pi' \in \bar{\Pi}}\, \|T^d\< V^{\pi} - T_{\pi'} \< V^{\pi} \|_{1,\<\rho_1},  
\end{align*}
two terms that are closely related to the notion of \emph{inherent Bellman error} \citep{antos2008learning,munos2008finite,lazaric2016analysis,scherrer2015approximate,haskell2017empirical}. 
Also, let $B_\gamma = \sum_{i,j=0}^\infty \gamma^{i+j} \< A_{i+j}
$, which was assumed to be finite in Assumption \ref{as:conc}.
\begin{theorem}
Suppose the sample size requirements of Lemma \ref{lem:loss} are satisfied with $\epsilon/B_\gamma$ and $\delta/K$ replacing $\epsilon$ and $\delta$, respectively. Then, the suboptimality of the policy $\pi_K$ can be bounded as follows:
\begin{align*}
\|V^* - V^{\pi_K} \|_{1,\<\nu} \le \,  &B_\gamma\< [ B_\gamma'\, \mathbb D_0(\bar{\Pi},\bar{\mathcal F}) + \mathbb D_1^d(\bar{\Pi}) ]\\
&+ \gamma^{Kd} \, \|V^* - V^{\pi_0}\|_{\infty} + \epsilon,
\end{align*}
with probability at least $1- \delta$.
\label{thm:one}
\end{theorem}

\noindent \textbf{Search Depth.} How should the search depth $d$ be chosen? Theorem \ref{thm:one} shows that as $d$ increases, fewer iterations $K$ are needed to achieve a given accuracy; however, the effort required of tree search (i.e., the function $m(\epsilon, \delta)$) grows exponentially in $d$. At the other extreme ($d=1$), more iterations $K$ are needed and the ``fixed cost'' of each iteration of the algorithm (i.e., sampling, regression, and classification --- all of the steps that do not depend on $d$) becomes more prominent. For a given problem and algorithm parameters, these computational costs can each be estimated and Theorem \ref{thm:one} can serve as a guide to selecting an optimal $d$.

\section{Case Study: King of Glory MOBA AI}
We implemented \textit{Feedback-Based Tree Search} within a new and challenging environment, the recently popular MOBA game \textit{King of Glory} by Tencent (the game is also known as \textit{Honor of Kings} and a North American release of the game is titled \textit{Arena of Valor}). Our implementation of the algorithm is one of the first attempts to design an AI for the 1v1 version of this game.

\begin{figure}[t!]
\centering
	\includegraphics[width=0.45\textwidth]{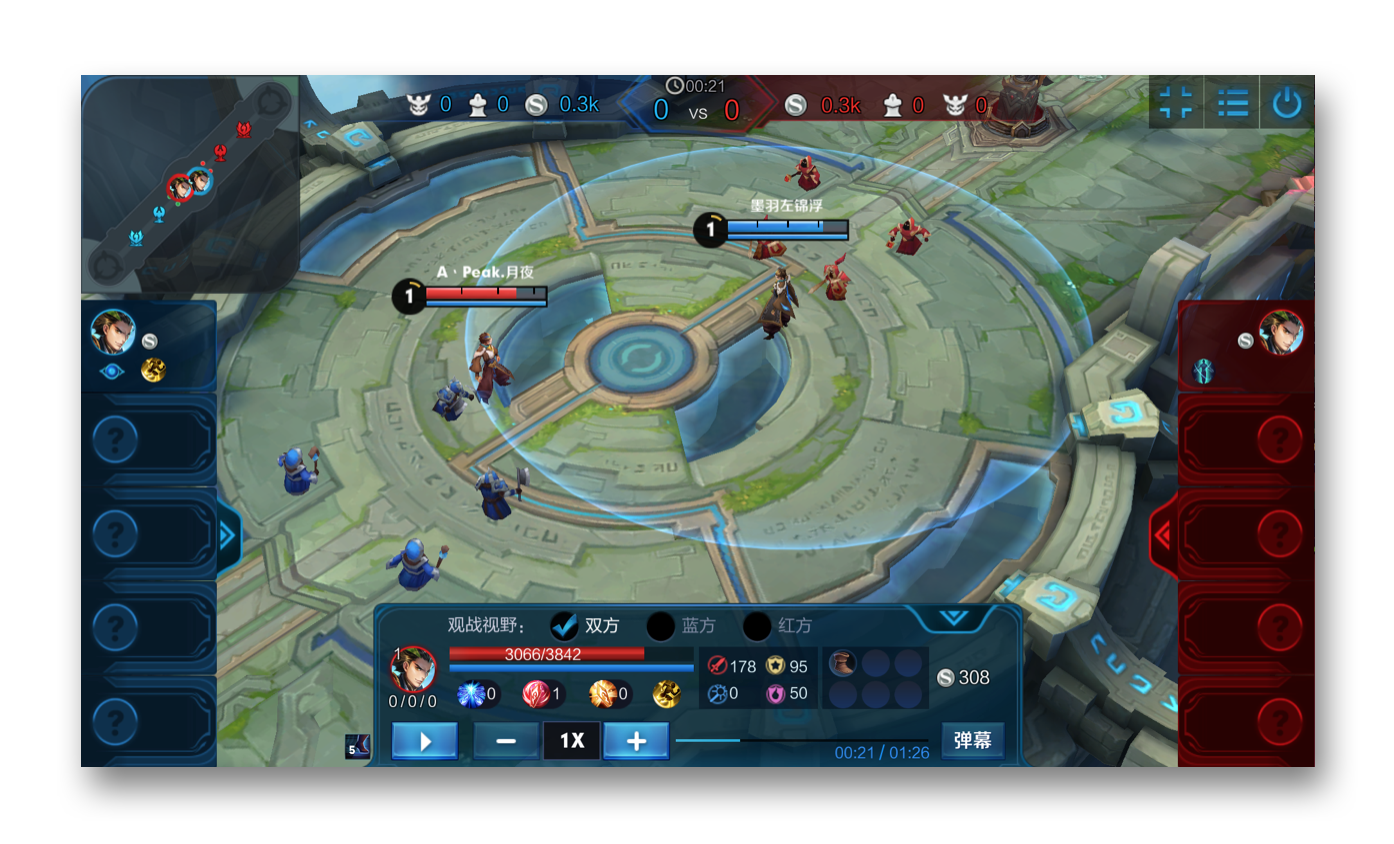}
		\vspace{-10pt}

	\caption{Screenshot from 1v1 \textit{King of Glory}	\vspace{-10pt}}
	\label{fig:screen}
\end{figure}

\begin{figure*}[b!]
	\makebox[\textwidth][c]{\includegraphics[scale=0.7]{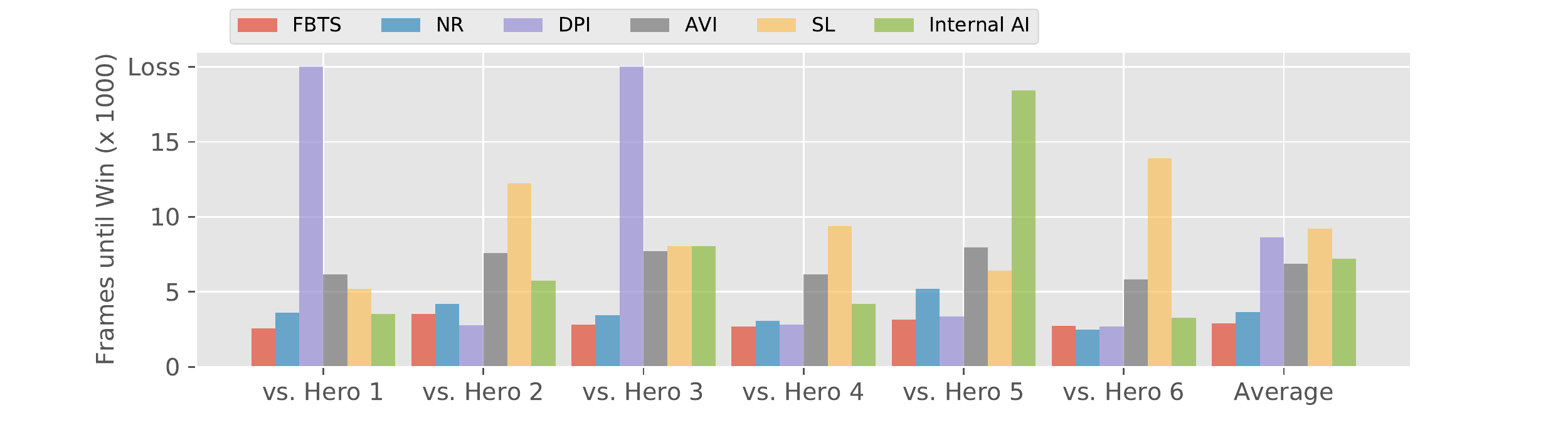}}
	\vspace{-22pt}
	\caption{Number of Frames to Defeat Marksman Heroes\vspace{-10pt}}
	\label{fig:frames}
\end{figure*}

\textbf{Game Description.} In the \emph{King of Glory}, players are divided into two opposing teams and each team has a base located on the opposite corners of the game map (similar to other MOBA games, like \textit{League of Legends} or \textit{Dota 2}). The bases are guarded by towers, which can attack the enemies when they are within a certain attack range. The goal of each team is to overcome the towers and eventually destroy the opposing team's ``crystal,'' located at the enemy's base.
For this paper, we only consider the 1v1 mode, where each player controls a primary ``hero'' alongside less powerful game-controlled characters called ``minions.'' These units guard the path to the crystal and will automatically fire (weak) attacks at enemies within range. 
Figure \ref{fig:screen} shows the two heroes and their minions; the upper-left corner shows the map, with the blue and red markers pinpointing the towers and crystals.

\textbf{Experimental Setup.} 
The state variable of the system is taken to be a 41-dimensional vector containing information obtained directly from the game engine, including \textit{hero locations}, \textit{hero health}, \textit{minion health}, \textit{hero skill states}, and \textit{relative locations to various structures}. There are 22 actions, including move, attack, heal, and special skill actions, some of which are associated with (discretized) directions. The reward function is designed to mimic \textit{reward shaping} \citep{ng1999policy} and uses a combination of signals including \textit{health}, \textit{kills}, \textit{damage dealt}, and \textit{proximity to crystal}. We trained five King of Glory agents, using the hero \textit{DiRenJie}:
\begin{enumerate}
\item The ``FBTS'' agent is trained using our feedback-based tree search algorithm for $K=7$ iterations of 50 games each. The search depth is $d =7$ and rollout length is $h=5$. Each call to $\mathtt{MCTS}$ ran for 400 iterations.
\item The second agent is labeled ``NR'' for \emph{no rollouts}. It uses the same parameters as the FBTS agent except no rollouts are used. At a high level, this bears some similarity to the AlphaGo Zero algorithm \citep{silver2017mastering} in a batch setting.
\item The ``DPI'' agent uses the \emph{direct policy iteration} technique of \cite{lazaric2016analysis} for $K=10$ iterations. There is no value function and no tree search (due to computational limitations, more iterations are possible when tree search is not used).
\item We then have the ``AVI'' agent, which implements \emph{approximate value iteration} \citep{de2000existence,van2006performance,munos2007performance,munos2008finite} for $K=10$ iterations. This algorithm can be considered a batch version of DQN \citep{Mnih2013}.
\item Lastly, we consider an ``SL'' agent trained via \emph{supervised learning} on a dataset of approximately 100{,}000 state/action pairs of human gameplay data. Notably, the policy architecture used here is consistent with the previous agents.
\end{enumerate}

In fact, both the policy and value function approximations are consistent across all agents; they use fully-connected neural networks with five and two hidden layers, respectively, and SELU (scaled exponential linear unit) activation \citep{klambauer2017self}. The initial policy $\pi_0$ takes random actions: move (w.p. 0.5), directional attack (w.p. 0.2), or a special skill (w.p. 0.3). Besides biasing the move direction toward the forward direction, no other heuristic information is used by $\pi_0$. $\mathtt{MCTS}$ was chosen to be a variant of UCT \citep{Kocsis2006} that is more amenable toward parallel simulations: instead of using the argmax of the UCB scores, we sample actions according to the distribution obtained by applying softmax to the UCB scores. 

In the practical implementation of the algorithm, $\mathtt{Regress}$ uses a cosine proximity loss while $\mathtt{Classify}$ uses a negative log-likelihood loss, differing from the theoretical specifications. Due to the inability to ``rewind'' or ``fast-forward'' the game environment to arbitrary states, the sampling distribution $\rho_0$ is implemented by first taking random actions (for a random number of steps) to arrive at an initial state and then following $\pi_k$ until the end of the game. To reduce correlation during value approximation, we discard $2/3$ of the states encountered in these trajectories. For $\rho_1$, we follow the $\mathtt{MCTS}$ policy while occasionally injecting noise (in the form of random actions and random switches to the default policy) to reduce correlation. During rollouts, we use the internal AI for the hero \textit{DiRenJie} as the opponent.

\textbf{Results.} As the game is nearly deterministic, our primary methodology for testing to compare the agents' effectiveness against a common set of opponents chosen from the internal AIs. We also added the internal DiRenJie AI as a ``sanity check'' baseline agent. To select the test opponents, we played the internal DiRenJie AI against other internal AIs (i.e., other heroes) and selected six heroes of the \textit{marksman} type that the internal DiRenJie AI is able to defeat. Each of our agents, including the internal DiRenJie AI, was then played against every test opponent. Figure \ref{fig:frames} shows the length of time, measured in frames, for each agent to defeat the test opponents (a value of 20{,}000 frames is assigned if the opponent won). Against the set of common opponents, FBTS significantly outperforms DPI, AVI, SL, and the internal AI. However, FBTS only slightly outperforms NR on average (which is perhaps not surprising as NR is the only other agent that also uses MCTS). Our second set of results help to visualize head-to-head battles played between FBTS and the four baselines (all of which are won by FBTS): Figure \ref{fig:ratios} shows the ratio of the FBTS agent's gold to its opponent's gold as a function of time. Gold is collected throughout the game as heroes deal damage and defeat enemies, so a ratio above 1.0 (above the red region) indicates good relative performance by FBTS. As the figure shows, each game ends with FBTS achieving a gold ratio in the range of $[1.25, 1.75]$.

\begin{figure}
	\makebox[0.49\textwidth][c]{\includegraphics[scale=0.7]{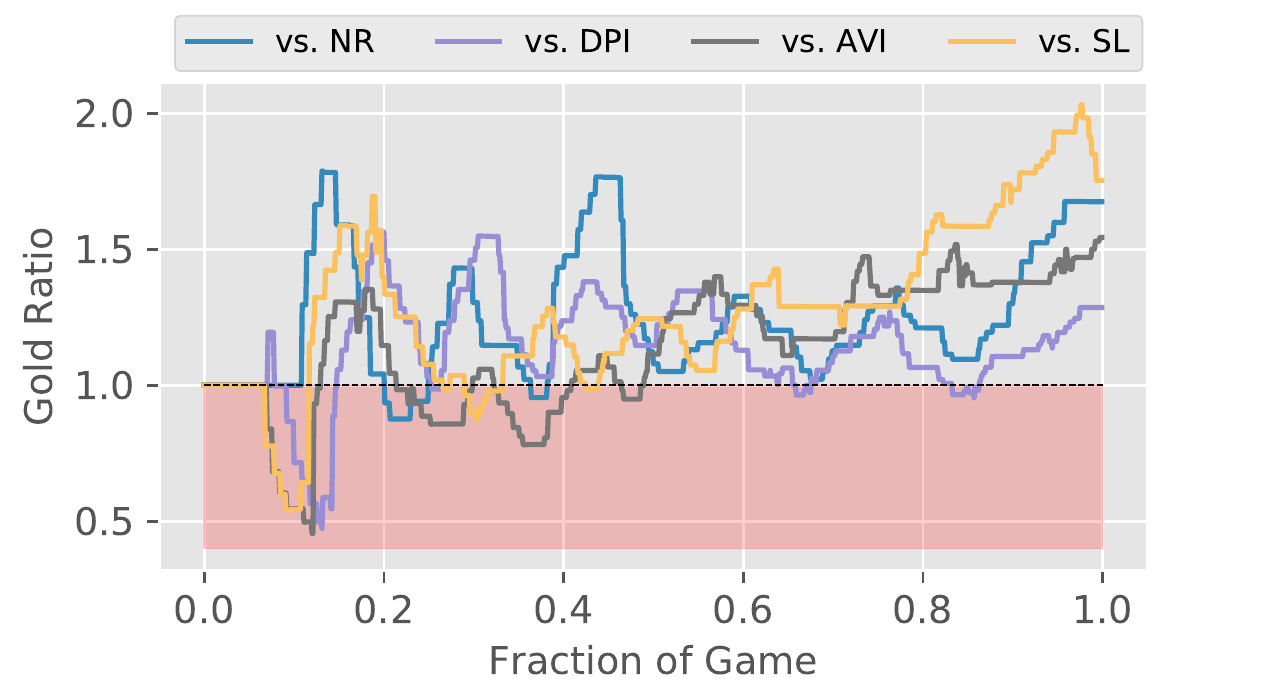}}
	\vspace{-20pt}
	\caption{In-game Behavior\vspace{-10pt}}
	\label{fig:ratios}
\end{figure}

\section{Conclusion \& Future Work}
In this paper, we provide a sample complexity analysis for feedback-based tree search, an RL algorithm based on repeatedly solving finite-horizon subproblems using MCTS. Our primary methodological avenues for future work are (1) to analyze a self-play variant of the algorithm and (2) to consider related techniques in multi-agent domains (see, e.g., \citet{hu2003nash}). The implementation of the algorithm in the 1v1 MOBA game \textit{King of Glory} provided us encouraging results against several related algorithms; however, significant work remains for the agent to become competitive with humans.

\FloatBarrier
\clearpage
\section*{Acknowledgements}
We sincerely appreciate the helpful feedback from four anonymous reviewers, which helped to significantly improve the paper. We also wish to thank our colleagues at Tencent AI Lab, particularly Carson Eisenach and Xiangru Lian, for assistance with the test environment and for providing the SL agent. The first author is very grateful for the support from Tencent AI Lab through a faculty award.

\bibliography{feedback}
\bibliographystyle{icml2018}



\section*{Supplementary material}

Supplementary Material for
# Feedback-Based Tree Search for Reinforcement Learning

Daniel R. Jiang, Emmanuel Ekwedike, Han Liu

## A   Outline

Section B contains the proofs for the results stated in the main paper.

- The proofs of Lemma 1 and Lemma 2 from the main paper are given in Sections B.1 and B.2. These two lemmas are important in that they provide the main structure for the sample complexity analysis. The bounds hold pointwise.

- In Section B.3, we provide some additional lemmas that are omitted from the main paper.

- Section B.4 gives the proof of Lemma 3 from the main paper, which makes use of Lemma 2 and the results from Section B.3.

- We prove the main result, Theorem 1, in Section B.5.

Lastly, in Section C, we provide additional implementation details regarding the neural network architecture, state features, and computation.

# B  Proofs

## B.1  Proof of Lemma 1

This proof is a modification of arguments used in [Lazaric et al., 2016, Equation 8 and Theorem 7]. By the fixed point property of $T_{\pi_{k+1}}$ and the definition of the Bellman operator $T$, we have $V^{\pi_k} - V^{\pi_{k+1}} \leq T^d V^{\pi_k} - T_{\pi_{k+1}} V^{\pi_{k+1}}$. Subtracting and adding $T_{\pi_{k+1}} V^{\pi_k}$:

$$V^{\pi_k} - V^{\pi_{k+1}} \leq T^d V^{\pi_k} - T_{\pi_{k+1}} V^{\pi_k} + T_{\pi_{k+1}} V^{\pi_k} - T_{\pi_{k+1}} V^{\pi_{k+1}}$$
$$\leq T^d V^{\pi_k} - T_{\pi_{k+1}} V^{\pi_k} + (\gamma P_{\pi_{k+1}})(V^{\pi_k} - V^{\pi_{k+1}}). \quad \text{(B.1)}$$

Similarly, we will bound the difference between $V^*$ and $V^{\pi_{k+1}}$ in terms of the distances between $V^* - V^{\pi_k}$ and $V^{\pi_k} - V^{\pi_{k+1}}$:

$$V^* - V^{\pi_{k+1}} \leq (\gamma P_{\pi^*})^d (V^* - V^{\pi_k}) + T^d V^{\pi_k} - T_{\pi_{k+1}} V^{\pi_k} + (\gamma P_{\pi_{k+1}})(V^{\pi_k} - V^{\pi_{k+1}}). \quad \text{(B.2)}$$

Using the bound $V^{\pi_k} - V^{\pi_{k+1}} \leq [I - (\gamma P_{\pi_{k+1}})]^{-1} (T^d V^{\pi_k} - T_{\pi_{k+1}} V^{\pi_k})$ from (B.1) on the last term of the right side of (B.2) along with a power series expansion on the inverse, we obtain:

$$V^* - V^{\pi_{k+1}} \leq (\gamma P_{\pi^*})^d (V^* - V^{\pi_k}) + \left[ I + (\gamma P_{\pi_{k+1}}) \sum_{j=0}^{\infty} (\gamma P_{\pi_{k+1}})^j \right] (T^d V^{\pi_k} - T_{\pi_{k+1}} V^{\pi_k})$$
$$= (\gamma P_{\pi^*})^d (V^* - V^{\pi_k}) + \left[ I - (\gamma P_{\pi_{k+1}}) \right]^{-1} (T^d V^{\pi_k} - T_{\pi_{k+1}} V^{\pi_k}),$$

which can be iterated to show:

$$V^* - V^{\pi_K} \leq (\gamma P_{\pi^*})^{Kd} (V^* - V^{\pi_0}) + \sum_{k=1}^{K} (\gamma P_{\pi^*})^{(K-k)d} \left[ I - (\gamma P_{\pi_k}) \right]^{-1} (T^d V^{\pi_{k-1}} - T_{\pi_k} V^{\pi_{k-1}}).$$

The statement from the lemma follows from taking absolute value, bounding by the maximum norm, and integrating.

## B.2  Proof of Lemma 2

For part (a), we note the following:

$$\|T_\pi V - T_\pi V^\mu\|_{1,\rho_1} = \gamma \int_{\mathcal{S}} \left| (P_\pi V)(s) - (P_\pi V^\mu)(s) \right| \rho_1(ds)$$
$$\leq \gamma \int_{\mathcal{S}} \left| V(s) - V^\mu(s) \right| \frac{d(\rho_1 P_\pi)}{d\rho_0} \rho_0(ds)$$
$$\leq \gamma \left\| \frac{d(\rho_1 P_\pi)}{d\rho_0} \right\|_{\infty} \|V - V^\mu\|_{1,\rho_0}.$$

By the concentrability conditions of Assumption 5, the right-hand-side can be bounded by $\gamma A'_1 \|V - V^\mu\|_{1,\rho_0}$. Now, we can apply the same steps with the roles of $T_\pi V$ and $T_\pi V^\mu$ reversed to see that the same inequality holds for $\|T_\pi V - T_\pi V^\mu\|_{1,\rho_1}$ and part (a) is complete.

For part (b), we partition the state space $\mathcal{S}$ into two sets:

$$\mathcal{S}^+ = \big\{ s \in \mathcal{S} : (T^d J)(s) \geq (T^d V^\mu)(s) \big\} \quad \text{and} \quad \mathcal{S}^- = \big\{ s \in \mathcal{S} : (T^d J)(s) < (T^d V^\mu)(s) \big\}.$$

We start with $\mathcal{S}^+$. Consider the finite-horizon $d$-stage MDP with terminal value $J$ and the same dynamics as our infinite-horizon MDP of interest. Let $\pi_1^J, \pi_2^J, \ldots, \pi_d^J$ be the time-dependent optimal policy for this MDP. Thus, we have

$$T_{\pi_1^J} T_{\pi_2^J} \cdots T_{\pi_d^J} J = T^d J \quad \text{and} \quad T_{\pi_1^J} T_{\pi_2^J} \cdots T_{\pi_d^J} V^\mu \leq T^d V^\mu.$$

Using similar steps as for part (a), the following hold:

$$\int_{\mathcal{S}^+} \big[ (T^d J)(s) - (T^d V^\mu)(s) \big] \, \rho_1(ds) \leq \int_{\mathcal{S}^+} \big[ (T^d T_\mu^h V)(s) - (T_{\pi_1^J} T_{\pi_2^J} \cdots T_{\pi_d^J} T_\mu^h V^\mu)(s) \big] \, \rho_1(ds)$$

$$\leq \gamma^{d+h} \int_{\mathcal{S}^+} \big| V(s) - V^\mu(s) \big| \frac{d(\rho_1 P_{\pi_1^J} P_{\pi_2^J} \cdots P_{\pi_d^J} P_\mu^h)}{d\rho_0} \, \rho_0(ds)$$

$$\leq \gamma^{d+h} A'_{d+h} \int_{\mathcal{S}^+} \big| V(s) - V^\mu(s) \big| \, \rho_0(ds).$$

Now, using the optimal policy with respect to the $d$-stage MDP with terminal condition $V^\mu$, we can repeat these steps to show that

$$\int_{\mathcal{S}^-} \big[ (T^d V^\mu)(s) - (T^d J)(s) \big] \, \rho_1(ds) \leq \gamma^{d+h} A'_{d+h} \int_{\mathcal{S}^-} \big| V(s) - V^\mu(s) \big| \, \rho_0(ds).$$

Summing the two inequalities, we obtain:

$$\| T^d J - T^d V^\mu \|_{1,\rho_1} \leq \gamma^{d+h} A'_{d+h} \left[ \int_{\mathcal{S}^+} \big| V(s) - V^\mu(s) \big| \, \rho_0(ds) + \int_{\mathcal{S}^-} \big| V(s) - V^\mu(s) \big| \, \rho_0(ds) \right]$$

$$= \gamma^{d+h} A'_{d+h} \| V - V^\mu \|_{1,\rho_0},$$

which completes the proof. $\blacksquare$

### B.3    Additional Technical Lemmas

**Lemma B.1** (Section 4, Corollary 2 of Haussler [1992]). *Let $\mathcal{G}$ be a set of functions from $\mathcal{X}$ to $[0, B]$ with pseudo-dimension $d_\mathcal{G} < \infty$. Then for all $0 < \epsilon \leq B$, it holds that*

$$\mathbf{P}\left( \sup_{g \in \mathcal{G}} \left| \frac{1}{m} \sum_{i=1}^m g(X^{(i)}) - \mathbf{E}\big[ g(X) \big] \right| > \epsilon \right) \leq 8 \left( \frac{32 e B}{\epsilon} \log \frac{32 e B}{\epsilon} \right)^{d_\mathcal{G}} \exp\left( -\frac{\epsilon^2 m}{64 B^2} \right), \quad \text{(B.3)}$$

*where $X^{(i)}$ are i.i.d. draws from the distribution of the random variable $X$.*

**Lemma B.2.** *Consider a policy $\mu \in \Pi$ and suppose each $s^i$ is sampled i.i.d. from $\rho_0$. Define initial states $s_0^{ij} = s^i$ for all $j$. Analogous to Step 5 of the algorithm and Assumption 1, let:*

$$\hat{Y}(s^i) = \frac{1}{M_0} \sum_{j=1}^{M_0} \sum_{t=0}^\infty \gamma^t r(s_t^{ij}, \mu(s_t^{ij})) \quad \text{and} \quad V \in \underset{f \in \bar{\mathcal{F}}}{\arg\min} \, \frac{1}{N_0} \sum_{i=1}^{N_0} \big| f(s^i) - \hat{Y}(s^i) \big|.$$

*For $\delta \in (0,1)$ and $\epsilon \in (0, V_{\max})$, if the number of sampled states $N_0$ satisfies the condition:*

$$N_0 \geq \left(\frac{32\,V_{\max}}{\epsilon}\right)^2 \left[\log \frac{32}{\delta} + 2\mathrm{d}_{\bar{\mathcal{F}}} \log \frac{64\,e\,V_{\max}}{\epsilon}\right] =: \Gamma_{\mathrm{a}}(\epsilon, \delta),$$

*and the number of rollouts performed from each state $M_0$ satisfies:*

$$M_0 \geq 8 \left(\frac{V_{\max}}{\epsilon}\right)^2 \log \frac{8\,N_0}{\delta} =: \Gamma_{\mathrm{b}}(\epsilon, \delta),$$

*then we have the following bound on the error of the value function approximation:*

$$\|V - V^\mu\|_{1,\rho_0} \leq \min_{f \in \bar{\mathcal{F}}} \|f - V^\mu\|_{1,\rho_0} + \epsilon,$$

*with probabilty at least $1 - \delta$.*

*Proof.* Recall that the estimated value function $V$ satisfies

$$V \in \arg\min_{f \in \widehat{F}} \frac{1}{N_0} \sum_{i=1}^{N_0} \left| f(s^i) - \frac{1}{M_0} \sum_{j=1}^{M_0} \left[V^\pi(s_0^i) + \xi^j(s_0^i)\right] \right|,$$

where for each $i$, the terms $\xi^j(s_0^i)$ are i.i.d. mean zero error. The inner summation over $j$ is an equivalent way to write $\hat{Y}(s_0^i)$. Noting that the rollout results $V^\mu(s_0^i) + \xi^j(s_0^i) \in [0, V_{\max}]$, we have by Hoeffding's inequality followed by a union bound:

$$\mathbf{P}\left(\max_i \left|\hat{Y}(s_0^i) - V^\mu(s_0^i)\right| > \epsilon\right) \leq N_0\,\Delta_1(\epsilon, M_0), \tag{B.4}$$

where $\Delta_1(\epsilon, M_0) = 2\exp\left(-2M_0\,\epsilon^2/V_{\max}^2\right)$. Define the function

$$\Delta_2(\epsilon, N_0) = 8 \left(\frac{32\,e\,V_{\max}}{\epsilon} \log \frac{32\,e\,V_{\max}}{\epsilon}\right)^{\mathrm{d}_{\bar{\mathcal{F}}}} \exp\left(-\frac{\epsilon^2\,N_0}{64\,V_{\max}^2}\right),$$

representing the right-hand-side of the bound in Lemma [B.1] with $B = V_{\max}$ and $m = N_0$. Next, we define the loss minimizing function $f^* \in \arg\min_{f \in \bar{\mathcal{F}}} \|f - V^\mu\|_{1,\rho_0}$. By Lemma [B.1], the probabilities of the events

$$\left\{\left|\|V - V^\mu\|_{1,\rho_0} - \frac{1}{N_0} \sum_{i=1}^{N_0} \left|V(s^i) - V^\mu(s^i)\right|\right| > \frac{\epsilon}{4}\right\} \text{ and }$$

$$\left\{\left|\|f^* - V^\mu\|_{1,\rho_0} - \frac{1}{N_0} \sum_{i=1}^{N_0} \left|f^*(s^i) - V^\mu(s^i)\right|\right| > \frac{\epsilon}{4}\right\} \tag{B.5}$$

are each bounded by $\Delta_2(\epsilon/4, N_0)$. Also, it follows by the definition of $V$ that

$$\frac{1}{N_0} \sum_{i=1}^{N_0} \left|V(s^i) - \hat{Y}(s^i)\right| \leq \frac{1}{N_0} \sum_{i=1}^{N_0} \left|f^*(s^i) - \hat{Y}(s^i)\right|.$$

Therefore, using (B.4) twice and (B.5) once, we have by a union bound that the inequality

$$\|V - V^\mu\|_{1,\rho_0} \leq \min_{f \in \bar{\mathcal{F}}} \|f - V^\mu\|_{1,\rho_0} + \epsilon$$

happens with probability greater than $1 - 2 N_0 \Delta_1(\epsilon/4, M_0) - 2 \Delta_2(\epsilon/4, N_0)$. We then choose $N_0$ so that $\Delta_2(\epsilon/4, N_0) = \delta/4$ (following Haussler [1992], we utilize the inequality $\log(a \log a) < 2 \log(a/2)$ for $a \geq 5$). To conclude, we choose $M_0$ so that $\Delta_1(\epsilon/4, M_0) = \delta/(4N_0)$. □

**Lemma B.3** (Sampling Error). *Suppose* $|\mathcal{A}| = 2$ *and let* $\mathrm{d}_{\bar{\Pi}}$ *be the VC-dimension of* $\bar{\Pi}$. *Consider* $Z, V \in \mathcal{F}$ *and suppose each* $s^i$ *is sampled i.i.d. from* $\rho_1$. *Also, let* $w^j$ *be i.i.d. samples from the standard uniform distribution and* $g : \mathcal{S} \times \mathcal{A} \times [0,1] \to \mathcal{S}$ *be a transition function such that* $g(s, a, w)$ *has the same distribution as* $p(\cdot | s, a)$. *For* $\delta \in (0,1)$ *and* $\epsilon \in (0, V_{\max})$, *if the number of sampled states* $N_1$ *satisfies the condition:*

$$N_1 \geq 128 \left( \frac{V_{\max}}{\epsilon} \right)^2 \left[ \log \frac{8}{\delta} + \mathrm{d}_{\bar{\Pi}} \log \frac{e N_1}{\mathrm{d}_{\bar{\Pi}}} \right] =: \Gamma_c(\epsilon, \delta, N_1),$$

*and the number of sampled transitions* $L_1$ *satisfies:*

$$L_1 \geq 128 \left( \frac{V_{\max}}{\epsilon} \right)^2 \left[ \log \frac{8}{\delta} + \mathrm{d}_{\bar{\Pi}} \log \frac{e L_1}{\mathrm{d}_{\bar{\Pi}}} \right] =: \Gamma_d(\epsilon, \delta, L_1),$$

*then we have the bounds:*

(a) $\displaystyle \sup_{\pi \in \bar{\Pi}} \left| \frac{1}{N_1} \sum_{i=1}^{N_1} |Z(s^i) - (T_\pi V)(s^i)| - \|Z - T_\pi V\|_{1, \rho_1} \right| \leq \epsilon \quad w.p.\ at\ least\ 1 - \delta.$

(b) $\displaystyle \sup_{\pi \in \bar{\Pi}} \left| \frac{1}{L_1} \sum_{j=1}^{L_1} \left[ r(s^i, \pi(s^i)) + \gamma V\left(g(s^i, \pi(s^i), w^j)\right) \right] - (T_\pi V)(s^i) \right| \leq \epsilon \quad w.p.\ at\ least\ 1 - \delta.$

*Proof.* We remark that in both (a) and (b), the term within the absolute value is bounded between 0 and $V_{\max}$. A second remark is that we reformulated the problem using $w^j$ to take advantage of the fact that these random samples do not depend on the policy $\pi$. Such a property is required to invoke [Györfi et al., 2006, Theorem 9.1], a result that [Lazaric et al., 2016, Lemma 3] depends on. With these two issues in mind, an argument similar to the proof of [Lazaric et al., 2016, Lemma 3] gives the conclusion for both (a) and (b). □

### B.4 Proof of Lemma 3

On each iteration of the the algorithm, two random samples are used: $\mathcal{S}_{0,k}$ and $\mathcal{S}_{1,k}$. From $\mathcal{S}_{0,k}$, we obtain $V_k$ and from $\mathcal{S}_{1,k}$ we obtain $\pi_{k+1}$. Let $\mathcal{S}_k = (\mathcal{S}_{0,k}, \mathcal{S}_{1,k})$ represent both of the samples at iteration $k$. We define:

$$\mathcal{G}_{k-1} = \sigma\{\mathcal{S}_1, \mathcal{S}_2, \ldots, \mathcal{S}_{k-1}\} \quad \text{and} \quad \mathcal{G}'_{k-1} = \sigma\{\mathcal{S}_1, \mathcal{S}_2, \ldots, \mathcal{S}_{k-1}, \mathcal{S}_{0,k}\}.$$

Due to the progression of the algorithm with two random samples per iteration, we will analyze each iteration in two steps. We first separate the two random samples by noting

that:

$$
\begin{aligned}
\|T^d V^{\pi_k} - T_{\pi_{k+1}} V^{\pi_k}\|_{1,\rho_1} &\leq \|T^d V^{\pi_k} - T^d J_k\|_{1,\rho_1} + \|T_{\pi_{k+1}} V_k - T_{\pi_{k+1}} V^{\pi_k}\|_{1,\rho_1} \\
&\quad + \|T^d J_k - T_{\pi_{k+1}} V_k\|_{1,\rho_1} \\
&\leq (\gamma A_1' + \gamma^{d+h} A_{d+h}') \|V_k - V^{\pi_k}\|_{1,\rho_0} + \|T^d J_k - T_{\pi_{k+1}} V_k\|_{1,\rho_1},
\end{aligned}
\tag{B.6}
$$

where the first inequality follows by adding and subtracting terms and the triangle inequality while the second inequality follows by Lemma 2. Now, we may analyze the first term on the right-hand-side conditional on $\mathcal{G}_{k-1}$ and the second term conditional on $\mathcal{G}_{k-1}'$.

As it is currently stated, Lemma B.2 gives an unconditional probability for a fixed policy $\mu$. However, since $\mathcal{S}_{0,k}$ is independent from $\mathcal{G}_{k-1}$ and $\pi_k$ is $\mathcal{G}_{k-1}$-measurable, we can utilize Lemma B.2 in a conditional setting using a well-known property of conditional expectations [Resnick, 2013, Property 12, Section 10.3]. This property will be repeatedly used in this proof (without further mention). We obtain that for a sample size $N_0 \geq \Gamma_{\mathrm{a}}(\epsilon'/(\gamma A_1' + \gamma^{d+h} A_{d+h}'), \delta')$,

$$
\mathbf{P}\Big(\|V_k - V^{\pi_k}\|_{1,\rho_0} > \min_{f \in \bar{\mathcal{F}}} \|f - V^{\pi_k}\|_{1,\rho_0} + \epsilon'/(\gamma A_1' + \gamma^{d+h} A_{d+h}') \,\big|\, \mathcal{G}_{k-1}\Big) \leq \delta'.
\tag{B.7}
$$

It remains for us to analyze the error of the second term $\|T^d J_k - T_{\pi_{k+1}} V_k\|_{1,\rho_1}$. By part (a) of Lemma B.3 with $Z = T^d J_k$ and $V = V_k$, if $N_1 \geq \Gamma_{\mathrm{c}}(\epsilon', \delta', N_1)$ and $s^i$ are sampled i.i.d. from $\rho_1$, we have

$$
\mathbf{P}\left(\left|\frac{1}{N_1}\sum_{i=1}^{N_1} |(T^d J_k)(s^i) - (T_{\pi_{k+1}} V_k)(s^i)| - \|T^d J_k - T_{\pi_{k+1}} V_k\|_{1,\rho_1}\right| > \epsilon' \,\Big|\, \mathcal{G}_{k-1}'\right) \leq \delta'.
\tag{B.8}
$$

The term $(T_{\pi_{k+1}} V_k)(s^i)$ is approximated using $L_1$ samples. Part (b) of Lemma B.3 along with a union bound shows that if $L_1 \geq \Gamma_{\mathrm{d}}(\epsilon', \delta'/N_1, L_1)$, then

$$
\mathbf{P}\Big(\max_i |\hat{Q}_k(s^i, \pi_{k+1}(s^i)) - (T_{\pi_{k+1}} V_k)(s^i)| > \epsilon' \,\big|\, \mathcal{G}_{k-1}'\Big) \leq \delta'.
\tag{B.9}
$$

Similarly, by Assumption 3, if the number of iterations of MCTS $M_1$ exceeds $m(\epsilon', \delta'/N_1)$, we can take a union bound to arrive at

$$
\mathbf{P}\Big(\max_i |\hat{U}_k(s^i) - (T^d J_k)(s^i)| > \epsilon' \,\big|\, \mathcal{G}_{k-1}\Big) \leq \delta'.
\tag{B.10}
$$

The maximum over $i$ can be replaced with an average over the $N_1$ samples and the conclusion of the last two bounds would remain unchanged. Since $\pi_{k+1}$ is assumed to optimize a quantity involving $\hat{U}_k$ and $\hat{Q}_k$, we want to relate this back to $\|T^d J_k - T_{\pi_{k+1}} V_k\|_{1,\rho_1}$. Indeed, taking expectation of both sides of inequalities (B.8)–(B.10) and then combining, we obtain that with probability at least $1 - 3\delta'$,

$$
\|T^d J_k - T_{\pi_{k+1}} V_k\|_{1,\rho_1} \leq \frac{1}{N_1}\sum_{i=1}^{N_1} |\hat{U}_k(s^i) - \hat{Q}_k(s^i, \pi_{k+1}(s^i))| + 3\epsilon'
$$

$$\leq \frac{1}{N_1} \sum_{i=1}^{N_1} \left| \hat{U}_k(s^i) - \hat{Q}_k(s^i, \tilde{\pi}(s^i)) \right| + 3\epsilon'$$

where $\tilde{\pi} \in \arg\min_{\pi \in \bar{\Pi}} \|T^d V^{\pi_k} - T_\pi V^{\pi_k}\|_{1,\rho_1}$. Following the same steps in reverse , we have:

$$\|T^d J_k - T_{\pi_{k+1}} V_k\|_{1,\rho_1} \leq \min_{\pi \in \bar{\Pi}} \|T^d V^{\pi_k} - T_\pi V^{\pi_k}\|_{1,\rho_1} + 6\epsilon', \tag{B.11}$$

with probability at least $1 - 6\delta'$. Finally, we take expectation of both sides of (B.7) and then combine with (B.6) and (B.11) while setting $\epsilon' = \epsilon/7$ and $\delta' = \delta/7$ to obtain

$$\|T^d V^{\pi_k} - T_{\pi_{k+1}} V^{\pi_k}\|_{1,\rho_1} \leq (\gamma A_1' + \gamma^{d+h} A_{d+h}') \min_{f \in \mathcal{F}} \|f - V^{\pi_k}\|_{1,\rho_0}$$
$$+ \min_{\pi \in \bar{\Pi}} \|T^d V^{\pi_k} - T_\pi V^{\pi_k}\|_{1,\rho_1} + \epsilon$$

with probability at least $1 - \delta$.

## B.5   Proof of Theorem 1

This proof synthesizes the previous lemmas. From the definition of $D_0(\bar{\Pi}, \bar{\mathcal{F}})$ and $\mathbb{D}_1^d(\bar{\Pi})$ from the main paper, we note that if the sample size assumptions of Lemma 3 are satisfied,

$$\|T^d V^{\pi_k} - T_{\pi_{k+1}} V^{\pi_k}\|_{1,\rho_1} \leq B_\gamma' \mathbb{D}_0(\bar{\Pi}, \bar{\mathcal{F}}) + \mathbb{D}_1^d(\bar{\Pi}) + \epsilon, \tag{B.12}$$

with probability at least $1 - \delta$. This removes any dependence on the iteration $k$ from the right-hand-side. We now integrate all results with Lemma 1 in order to find a bound on the suboptimality $\|V^* - V^{\pi_K}\|_{1,\nu}$. Consider the distribution $\Lambda_{\nu,k}$, as defined in Lemma 1, which needs to be related to $\nu$. We can use the power series expansion to write:

$$\Lambda_{\nu,k} = \nu \left( P_{\pi^*} \right)^{K-k} \sum_{i=0}^{\infty} (\gamma P_{\pi_k})^i.$$

For a fixed $i$, the measure $\nu$ is transformed by applying $\pi^*$ a total of $K - k$ times and then $\pi_k$ a total of $i$ times. We see that the summation term on the right-hand-side of Lemma 1 can be upper-bounded in the following way:

$$\sum_{k=1}^{K} \gamma^{K-k} \|T^d V^{\pi_{k-1}} - T_{\pi_k} V^{\pi_{k-1}}\|_{1,\Lambda_{\nu,k}}$$
$$\leq \left( \sum_{j=0}^{K-1} \sum_{i=0}^{\infty} \gamma^{j+i} A_{j+i} \right) \max_{k \leq K} \|T^d V^{\pi_{k-1}} - T_{\pi_k} V^{\pi_{k-1}}\|_{1,\rho_1},$$

where we use Assumption 5 with $m = K - k + i$, maximize over $k$ for the loss term, and then re-index with $j = K - k$. The coefficient in parentheses can be upper-bounded by $B_\gamma$ (since all $A_{j+i}$ are nonnegative). Finally, we use (B.12) and then a union bound over the $K$ iterations to conclude the statement of the theorem.

# C   Implementation Details

## C.1   Neural Network Architecture

The policy and value function approximations use fully-connected neural networks with five and two hidden layers, respectively, and SELU (scaled exponential linear unit) activation [Klambauer et al., 2017]. The policy network contains two sets of outputs: (1) one of seven actions (no action, normal attack, move, skill 1, skill 2, skill 3, and heal) and (2) a two-dimensional direction parameter used for the action. The first two hidden layers are shared and have 120 and 100 hidden units, while each of the two outputs corresponds to a set of three hidden layers with 80, 70, and 50 hidden units. The value function approximation uses a fully-connected network with 128 hidden units in the first layer and 96 hidden units in the second layer. As mentioned in the main paper, this architecture is consistent across all agents whenever policy and/or value networks are needed.

## C.2   Features of the State

As shown in Table 1, the state of the game is represented by 41-dimensional feature vector, which was constructed using the output from the game engine and API. The features consists of basic attributes of the two heroes, the computer-controlled units, and structures. The feature lists also have information on the relative positions of the other units and structures with respect to the hero controlled by algorithm.

Table 1: State Feature List

| No. | Feature | Dimensions |
|-----|---------|------------|
| 1 | Location of Hero 1 | 2 |
| 2 | Location of Hero 2 | 2 |
| 3 | HP of Hero 1 | 1 |
| 4 | HP of Hero 2 | 1 |
| 5 | Hero 1 skill cooldowns | 5 |
| 6 | Hero 2 skill cooldowns | 5 |
| 7 | Direction to enemy hero | 3 |
| 8 | Direction to enemy tower | 4 |
| 9 | Direction to enemy minion | 3 |
| 10 | Enemy tower HP | 1 |
| 11 | Enemy minion HP | 1 |
| 12 | Direction to the spring | 3 |
| 13 | Total HP of allied minions | 1 |
| 14 | Enemy's tower attacking Hero 1 | 3 |
| 15 | Hero 1 in range of enemy towers | 3 |
| 16 | Hero 2 in range of enemy towers | 3 |

## C.3   Tree Search Details

We provide some more information regarding the implementation of feedback-based tree search. A major challenge in implementing in *King of Glory* is that the game engine can only move forward, meaning that our sampled states are not i.i.d. and instead follow the trajectory of the policy induced by `MCTS`. However, to decrease the correlation between visited states, we inject random movements and random switches to the internal AI policy in order to move to a "more random" next state. Rollouts are performed on separate processors to enable tree search in a game engine that cannot rewind. All experiments use the `c4.2xlarge` instances on Amazon Web Services, and we utilized parallelization across four cores within each call to `MCTS`.



\end{document}